\documentclass[conference]{IEEEtran}
\IEEEoverridecommandlockouts

\usepackage[utf8]{inputenc}
\usepackage[english]{babel}

\usepackage[margin=0.7in]{geometry}
\usepackage{cite}
\usepackage{amsmath,amssymb,amsfonts}

\usepackage[ruled,vlined]{algorithm2e}

\usepackage{afterpage}
\usepackage{graphicx}
\usepackage{textcomp}
\usepackage{xcolor}
\usepackage{amsthm}
\usepackage{paralist}
\usepackage{mathtools, nccmath}
\usepackage{setspace}
\usepackage{makecell}
\usepackage[inline]{enumitem}
\usepackage{subcaption}
\usepackage{bbm}
\usepackage{mathtools}
\usepackage{dsfont}

\theoremstyle{definition}

\newtheorem{theorem}{Theorem}
\newtheoremstyle{exampstyle}
{1pt} 
{1pt} 
{} 
{} 
{\bfseries} 
{} 
{.5em} 
{} 

\theoremstyle{exampstyle} 
\theoremstyle{exampstyle} \newtheorem{remark}{Remark}
\theoremstyle{exampstyle} \newtheorem{definition}{Definition}
\theoremstyle{exampstyle} 
\theoremstyle{exampstyle} 

\newcommand{\minus}{\scalebox{0.75}[1.0]{$-$}}

\usepackage{xpatch}
\makeatletter
\xpatchcmd{\@thm}{\thm@headpunct{.}}{\thm@headpunct{}}{}{}
\makeatother

\def\BibTeX{{\rm B\kern-.05em{\sc i\kern-.025em b}\kern-.08em
		T\kern-.1667em\lower.7ex\hbox{E}\kern-.125emX}}
\begin{document}
	
	\title{Communication-Efficient Federated Distillation with Active Data Sampling
			\thanks{This work is supported in part by the Hong Kong Research Grant Council under Grant No. 16208921.}
	}
	
	\author{
		\IEEEauthorblockN{Lumin Liu, Jun Zhang, S. H. Song, and Khaled B. Letaief, \emph{Fellow, IEEE}}\\
		\IEEEauthorblockA{Dept. of ECE, The Hong Kong University of Science and Technology, Hong Kong\\
			Email:{ lliubb@ust.hk,
				eejzhang@ust.hk,
				eeshsong@ust.hk,
				eekhaled@ust.hk}}
	}               
	
	\maketitle
	\begin{abstract}
	    Federated learning (FL) is a promising paradigm to enable privacy-preserving deep learning from distributed data. Most previous works are based on federated average (FedAvg), which, however, faces several critical issues, including a high communication overhead and the difficulty in dealing with heterogeneous model architectures. Federated Distillation (FD) is a recently proposed alternative to enable communication-efficient and robust FL, which achieves orders of magnitude reduction of the communication overhead compared with FedAvg and is flexible to handle heterogeneous models at the clients. However, so far there is no unified algorithmic framework or theoretical analysis for FD-based methods. In this paper, we first present a generic meta-algorithm for FD and investigate the influence of key parameters through empirical experiments. Then, we verify the empirical observations theoretically. Based on the empirical results and theory, we propose a communication-efficient FD algorithm with active data sampling to improve the model performance and reduce the communication overhead. Empirical simulations on benchmark datasets will demonstrate that our proposed algorithm effectively and significantly reduces the communication overhead while achieving a satisfactory performance.
	\end{abstract}
	
	\section{Introduction}
	Federated Learning (FL) has recently attracted considerable attention due to its ability to collaboratively and effectively train machine learning models while preserving users' privacy \cite{FLsurvey}. A popular FL algorithm is Federated Average (FedAvg)  \cite{mcmahan2017communication}, which aggregates models trained by different clients via weight averaging. FedAvg has been successfully implemented on real-world applications \cite{hard2018federated} and has inspired tremendous research interests in designing efficient and robust FL algorithms \cite{ICCLiu}.
	
	Nevertheless, weight-averaging-based methods have many limitations. For example, the local neural network architectures at different clients have to be the same, and the communication overhead is proportional to the local model size. The communication issue has been partially addressed by adopting model compression techniques to reduce the communication cost \cite{reisizadeh2020fedpaq}, while the restrictions of model architectures have been largely ignored. In a realistic FL system, clients have heterogeneous computational and communication resources. Hence, it would be highly ineffective to require all the local models to be of the same architecture.
	
    To allow heterogeneous models and reduce the communication overhead, knowledge distillation (KD) was introduced to enable effective low-cost information exchange in FL. KD \cite{hinton2015distilling} is an effective mechanism to transfer knowledge from a large teacher model to a small student model, where the student model mimics the teacher model's output, i.e., logits, on the same training data. Thus, the model architecture of the student can be different and the communication cost only depends on the logits size rather than the model weights. 
    However, since KD is data-dependent, the training data were assumed to be universally accessible in classic KD methods. Considering the privacy regulation in FL, Federated Distillation (FD) needs to achieve distillation without sharing the local private data.
    
    In \cite{jeong2018communication}, distillation was achieved by transmitting and aggregating label-wise logits of the local training data. 
    In \cite{oh2020mix2fld}, an auxiliary distillation dataset was generated with a linear mixture of the local training data. However, the learning performance of these two approaches degrades noticeably compared with FedAvg. In \cite{dsflNishioMC}, it was assumed that there exists a public unlabeled dataset at both the server and the clients for the distillation process. An entropy reduction technique was proposed to improve the model performance under non-iid data. In \cite{sattler2020communication}, delta-coding on the logits was proposed to further reduce the communication cost and the knowledge was distilled at the server side. In \cite{NEURIPS2020_ensembledistillation}, distillation was introduced as an additional technique after weight averaging at the server side. 
    In  \cite{NEURIPS2020_ondevicedistillation}, fully distributed distillation in a connected network was considered and the gradient of the training loss function was proved to converge to zero asymptotically. 
    These approaches showed comparable or even better performance than the weight-averaging method with a much less communication cost and even in heterogeneous model architectures. 
    
    Existing FD algorithms, while sharing similar key steps, are proposed from different perspectives, which makes it difficult to characterize and improve their performance. 
    For FedAvg, systematic and theoretical understandings have been developed \cite{9563947},  which enables further design and optimization for the FL system with weight-averaging-based methods.
    However, for these FD algorithms, despite the empirical success, there lacks a clear understanding, either experimentally or theoretically, of the key components, i.e., \begin{enumerate*}
        \item the auxiliary data distribution; 
        \item the logits aggregation strategy; and
        \item the size of the upload logits.
    \end{enumerate*}
    
    In this paper, we endeavor to fill this important gap and answer these questions. We will first propose a generic meta-algorithm for FD, and investigate the effects of key parameters. Our results will show that in order to achieve a good training performance, the public auxiliary data distribution should be close to the local training data, the logits aggregation strategy should reduce the logits entropy, and the size of the upload logits size should be sufficiently large. To verify and better understand these observations, we will provide a theoretical characterization of the FD meta-algorithm with a binary classification problem and Gaussian mixture models. 
    
    Inspired by the findings from these empirical observations and theoretical results, i.e., the logit entropy should be low and the distillation set size should be large, we will propose a communication-efficient FD algorithm, named, Federated distillation with Active data Sampling (FAS). In the proposed algorithm, each client only uploads a subset of the logits with low entropy. Accordingly, the samples from the public data that join the distillation will be different among different users and thus the size of the distillation logits at the server size will be larger than the upload communication cost. Simulation results will demonstrate that the proposed algorithm achieves a better performance under a limited communication cost and non-iid data distribution compared with baseline FD methods.
    
	\section{Preliminary}
	In this section, we briefly introduce the notations for FL and KD, respectively.
	\subsection{Federated Learning}
	
	In FL, there are $n$ clients with local private datasets $\{ \mathcal{D}_i \}_{i=1}^n$ following the probability distribution $\{\mathcal{P}_i\}_{i=1}^n$. The dataset size of the $i$-th client is $D_i$. Based on the local dataset $\{ \mathcal{D}_i \}$, the empirical local loss function for the $i$-th client is expressed as
    \begin{equation}
	    L_i (\theta) = \frac{1}{D_i} \sum_{ {\{ \boldsymbol{x}_j, y_j \} } \in \mathcal{D}_i}\mathcal{L}(\theta,\boldsymbol{x}_j, y_j) \label{eq:4},
    \end{equation}
    where $\mathcal{L}(\theta,\boldsymbol{x}_j, y_j)$ is the loss function of the training data sample $\boldsymbol{x}_j$ and its label $y_j$, and $\theta$ denotes the model parameters. 
    The target in FL is to learn a global model that performs well on the average of the local data distributions. Denote the joint dataset as $\mathcal{D} = \bigcup_{i=1}^n \mathcal{D}_i $ then the target training loss function in FL is given by
    \begin{equation}
	    L (\theta) = \frac{1}{\sum_{i=1}^{n}D_i} \sum_{{\xi_j}\in \mathcal{D}}\mathcal{L}(\theta,\xi_j) = \frac{1}{\sum_{j=1}^{n}D_j} \sum_{i=1}^n D_i L_i(\theta). \label{eq:5}
    \end{equation}
    The most commonly adopted training algorithm in FL is FedAvg, where each client periodically updates its model locally and averages the local model parameters through communications with a central server (e.g., at the cloud or edge).
    The parameters of the local model on the $i$-th client after $t$ steps of stochastic gradient descent (SGD) iterations are denoted as $\theta_t^i$. In this case,  $\theta_t^i$ evolves as follows

    \begin{equation}
    \text{$\theta_t^i$ } = 
    \begin{cases}
    \text{$\theta_{t-1}^i - \eta  \tilde{\nabla} L_i(\theta_{t-1}^i)$} &  \text{$t \mid \tau \neq 0$}\\
    \text{ $\frac{1}{n} \sum_{i=1}^n[\theta_{t-1}^i - \eta \tilde{\nabla} L_i(\theta_{t-1}^i)]$ } &
    \text{$t \mid \tau = 0$}
    \end{cases} \label{eq:6}
    \end{equation}
	
	\subsection{Knowledge Distillation}
	\textit{Knowledge Distillation} (KD) is the process of distilling knowledge from a large and well-trained teacher model to a small student model. 
	For a classification problem with $N_c$ classes,
	the logit of data sample $\boldsymbol{x_i}$ is denoted as $t(\boldsymbol{x_i}) $ and it is the vector of the class probabilities which is obtained by using a softmax function on the neural network output. That is,
	\begin{equation}
		t(\boldsymbol{x_i}) = softmax (\theta(\boldsymbol{x_i}) ),
	\end{equation}
	where  $\theta(x_i) \in \mathcal{R}^{N_c}$ denotes the model output of input data sample $\boldsymbol{x_i}$, and $\theta(\cdot)$ is the neural network function parameterized by model parameters $\theta$.
	Speficically, for the $n$-th element of logit $t(\boldsymbol{x_i})$ of data sample $\boldsymbol{x_i}$,
	\begin{equation}
	    t^n(\boldsymbol{x}_i) = \frac{\exp{( \theta(\boldsymbol{x}_i)^n / T) }}{\sum_{m=1}^{N_c}\exp{(\theta(\boldsymbol{x}_i)^m /T)}  },
	\end{equation}
	where $T$ is the distillation temperature with a higher temperature producing a smoother probability distribution over classes.
	
	The distillation loss of the trainset $\mathcal{D}$ is the cross-entropy loss for the teacher logit $t_t$ and the student logit $t_s$, which is
	\begin{equation}
	   L_{distill} =  - \sum_{\boldsymbol{x} \in \mathcal{D}} \sum_{n=1}^{N_c}t_t^n (\boldsymbol{x}) log(t_s^n(\boldsymbol{x})).
	   \label{eq:1}
	\end{equation}
	In the distillation process, the student's objective function is an average of the distillation loss $L_{distll}$ and the cross entropy loss with the groudtruth labels.
	
		\begin{figure}[t] 
		\centerline{\includegraphics[width=.8\linewidth]{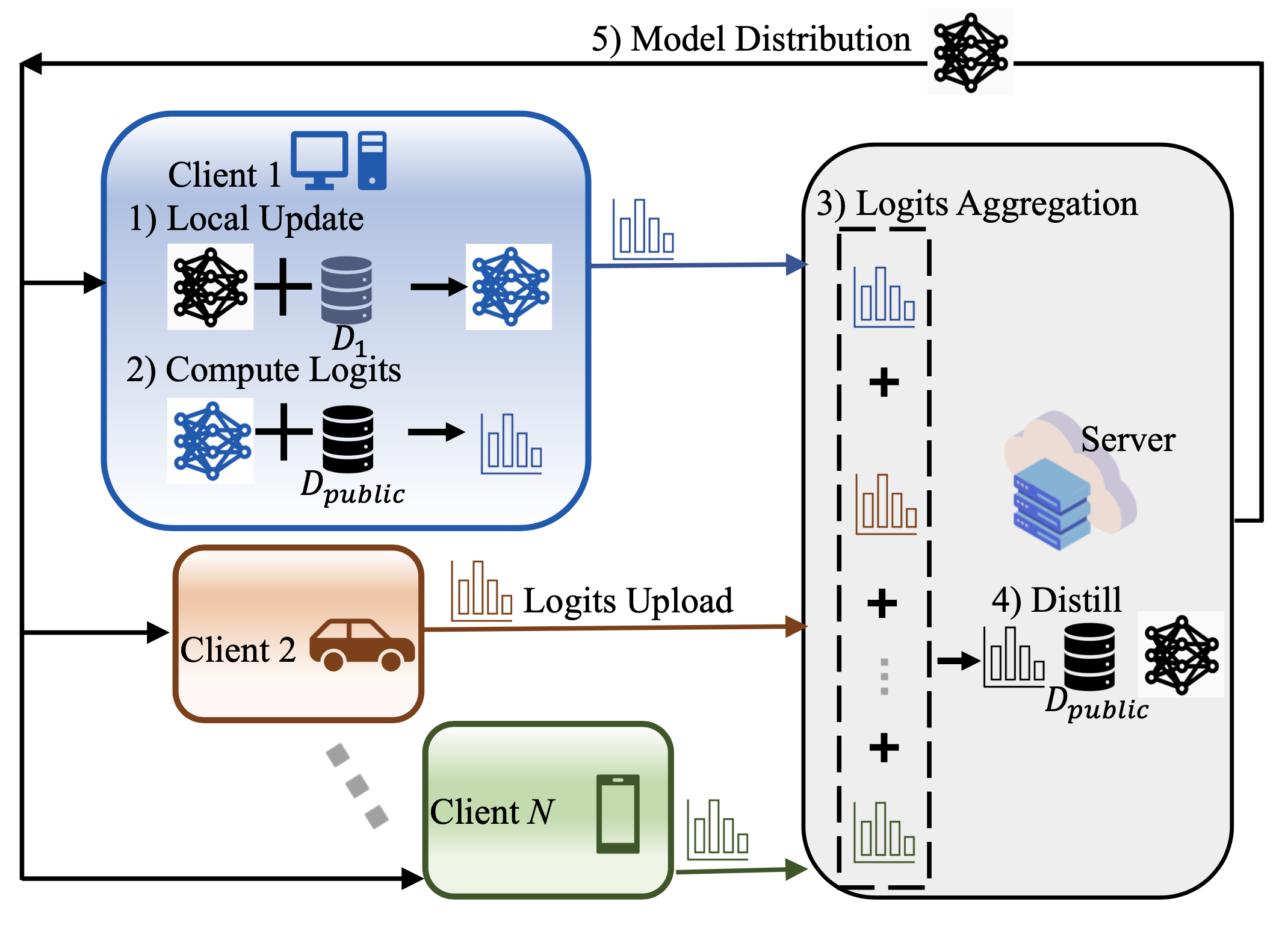}} 
		\caption{Illustration of FD meta-algorithm.}
		\label{fig1}
	\end{figure}
	
	\section{Federated Distillation Meta-Algorithm}
    In this section, we will first introduce the FD system and present a meta-algorithm, which is constituted of several key components. Then, we will investigate the impacts of these key components both empirically and theoretically.
    
    \begin{table*}
    	\centering
    	\caption{Comparison of different algorithms.}
    	\begin{tabular}{|c |c  | c |c | c| c| c| c|} 
    		\hline
    		&\small Upload &\small Aggregation   & \makecell{\small Auxiliary \\ \small Dataset} \normalsize & \makecell{\small Model \\ \small Heterogeneity }& \makecell{\small Communication \\ \small Cost (Uplink)} \normalsize & \makecell{\small Model \\ \small Performance}\\
    		\hline
    		FedAvg\cite{mcmahan2017communication} & \small Weights &\footnotesize Average & $\times$& $\times$ &$\mathcal{O}(|\theta|)$ & \small Baseline \\
    		\hline
    		FedDF\cite{NEURIPS2020_ensembledistillation} & \small Weights & \footnotesize Average \& Distill&\makecell{\checkmark} & \checkmark  & $\mathcal{O}(|\theta|)$ & \checkmark \\
    		\hline 
    		FDA\cite{jeong2018communication} & \makecell{\small Label-logits}  &\footnotesize Average  &$\times$& \checkmark & $N_c^2$ &$\times$ \\
    		\hline
    		DSFL\cite{dsflNishioMC} & \makecell{\small Logits}  &\footnotesize \makecell{Entropy Reduction \\ Average} & \makecell{
    			\makecell{\checkmark}} & \checkmark & $|\mathcal{D}_{logit}|N_c$ &\checkmark \\
    		\hline 
    		CEFD \cite{sattler2020communication} & \makecell{\small Delta-coded logits }  &\footnotesize Average& \makecell{\checkmark} & \checkmark & \textbackslash \footnotemark & \checkmark \\
    		\hline 
    		FD meta-algorithm & \small Logits  &\footnotesize Average\&Distill &\checkmark & \checkmark & $|\mathcal{D}_{logit}|N_c$ & \checkmark \\
    		\hline
    	\end{tabular}
    	\label{table3}
    \end{table*}
    
	\subsection{FD Meta-Algorithm}
	
    For a FD system with $n$ clients, the local private labeled dataset of the $i$-th client is denoted as $ \mathcal{D}_i =\{ \boldsymbol{x_i^j}, y_i^j \}_{j=1}^{D_i}$. A shared public unlabeled dataset $\mathcal{D}_{pub} = \{\boldsymbol{x^j} \}_{j=1}^{D_{pub}}$ is assumed accessible for each client and the server, where each data sample is identified by a unique and universal index.  The local loss function of client $i$ with local model parameters $\theta _i$ is denoted as $L_i(\theta_i)$. 	
   	In the $k$-th communication round, the selected clients perform local updates on their local private datasets $\mathcal{D}_i$'s and get locally trained models $\theta^i_k$'s.
   	
   	The weight-averaging-based method will directly upload and average the model weights of different clients, and then the training proceeds to the next communication round. However, since the local models $\{\theta_i \}_{i=1}^{n}$ may have different neural network architectures, e.g., simple fully-connected neural networks and ResNets, it is infeasible to directly average the model weights of these heterogeneous clients.
   	
   	\footnotetext{Since in CEFD, delta coding is applied to the logits of the whole distillation datasets $\mathcal{D}_{pub}$, the communication cost is smaller than $|\mathcal{D}_{pub}|N_c$. But it  varies in the training process.}
   	
   	\begin{algorithm}[t] 
   		\setstretch{1}
   		\SetAlgoLined
   		Initialize local model $\{\theta^i\}$ and server model $\theta$ \\
   		\For{k = 0,1,\dots, K-1}{
   			Download the server model $\theta_{k-1}^i  = \theta_{k-1}$,\\
   			Select clients $\mathcal{C}$ from the $n$ clients,\\
   			Select a subset $\mathcal{D}_{logit}$ of the public dataset $\mathcal{D}_{pub}$, \\
   			\For{client $i \in \mathcal{C}$}{
   				\textbf{Local update}: $\theta_k^i = \theta_{k-1}^i - \eta  \tilde{\nabla} L_i(\theta_{k-1}^i)$,\\
   				\textbf{Compute the logits}: $t_i(\xi) = softmax(\theta_t^i(\xi)) \text{ for }\xi \in \mathcal{D}_{logit}$ \\
   				\textbf{Upload the logits and indexes}: $\{t_i(\boldsymbol{x})) \}_{\boldsymbol{x} \in \mathcal{D}_{logit}}$, $\mathcal{I}_{logit}$\\
   			}	
   			\textbf{Aggregate the logits}: $t (\boldsymbol{x}) = \frac{1}{| \mathcal{C} |} \sum{_{i \in \mathcal{C}}} t_i(\boldsymbol{x})$\\
   			\textbf{Model distillation}: $\theta_t = \theta_{k-1} - \eta  \tilde{\nabla} L_{distill}(\theta_{k-1})$\\
   		}
   		\caption{FD Meta-Algorithm}
   		\label{algorithm1}
   	\end{algorithm}
   	
   	To enable information sharing of the clients with heterogeneous neural architectures, in FD, the selected clients will compute the logits on a subset $\mathcal{D}_{logit}$ of the public unlabeled dataset $\mathcal{D}_{pub}$, and the indexes of the data sample in $\mathcal{D}_{logit}$ are denoted as $\mathcal{I}_{logit}$.  The computed logits of the selected subset $\{t_i(\boldsymbol{x}) \}_{\boldsymbol{x} \in \mathcal{D}_{logit}}$ and the index $\mathcal{I}_{logit}$ are uploaded to the server for logits averaging. The averaged logits then serve as the teacher logits in the distillation loss in \eqref{eq:1}. The distilled model is then distributed back to the selected clients in the next communication round. The uploading communication cost is   $N_c*|\mathcal{D}_{logit}|$ and the downloading communication cost is proportional to the local model size, i.e., $\mathcal{O}(|\theta_i|)$.  	   	   	
   	The FD system and the detailed procedure of the algorithm are illustrated in Fig. \ref{fig1} and Algorithm \ref{algorithm1}, respectively.
	
	It is worth noting that in some existing works (e.g.,  \cite{dsflNishioMC}), the averaged logits are distributed to the clients and the distillation happens at the client side. 
	Local distillation reduces the downloading communication cost to $N_c* |\mathcal{D}_{logit}|$ and is completely free of the worry of the model heterogeneity. However, it also induces more local computation. In addition, partial client participation is not allowed if the averaged logits are sent back to clients. To allow a heterogeneous model for the local update, the server can distill the averaged weights into different models and then send back the weights to its corresponding client. We compared these two methods empirically and found that the weights downloading method exhibits a faster convergence. Thus, we will adopt the model weights downloading method. 
	Finally, the differences of the typical algorithms mentioned in this paper are summarized in Table \ref{table3}. Given the enormous size of deep learning models, distillation-based methods achieve orders of magnitude reduction in the communication overhead and allow heterogeneous models for the local update. Comparable model performance can be achieved with an auxiliary public dataset. The FD-meta algorithm concluded the key components for the FD-based methods and can be extended to the existing work \cite{dsflNishioMC,sattler2020communication} with slight modifications, e.g. in \cite{dsflNishioMC} the entropy of the averaged logits was reduced. With this FD meta-algorithm, we can better understand the design principles in a FD system.

	\subsection{Empirical Observations}
	\label{observations} 
	There are some key components in the FD system which influence the communication cost and the final learning performance, i.e., the data distribution of public dataset $\mathcal{D}_{pub}$, the logits aggregation method, and the upload logits size.
	In this section, we investigate these key components with the FD-meta algorithm. With extensive simulations on the CIFAR-10 dataset, we will show their impacts in the following.  
\begin{figure*}
	\centering
	\begin{subfigure}[b]{0.32\textwidth}
		\centering
		\includegraphics[width=\textwidth]{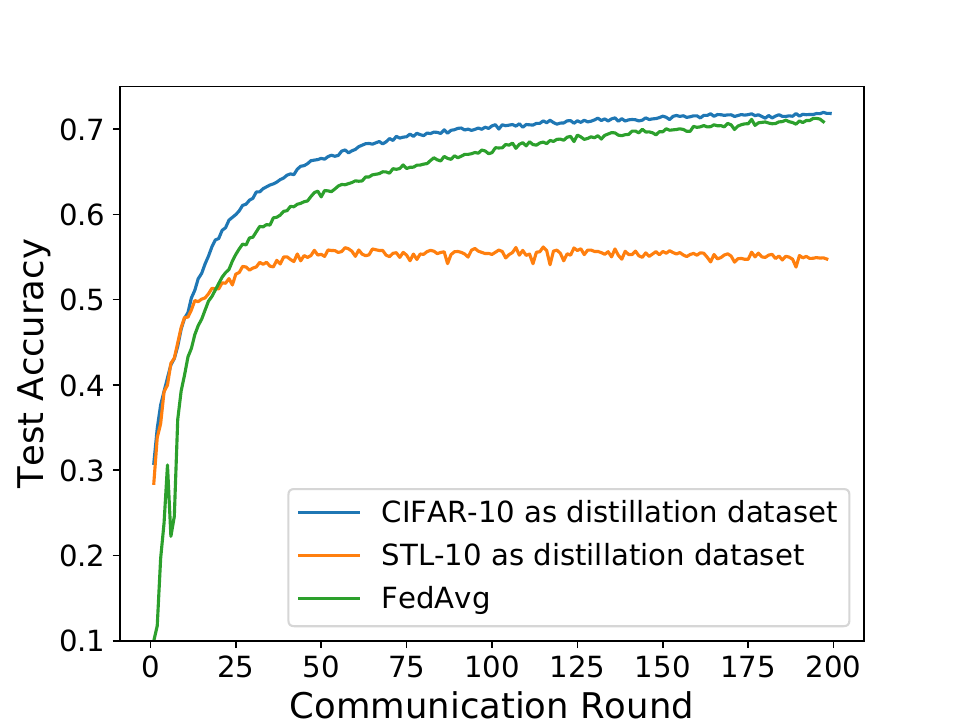}
		\caption{Distillation data distribution.}
		\label{fig:datadistribution}
	\end{subfigure}
	\hfill
	\begin{subfigure}[b]{0.32\textwidth}
		\centering
		\includegraphics[width=\textwidth]{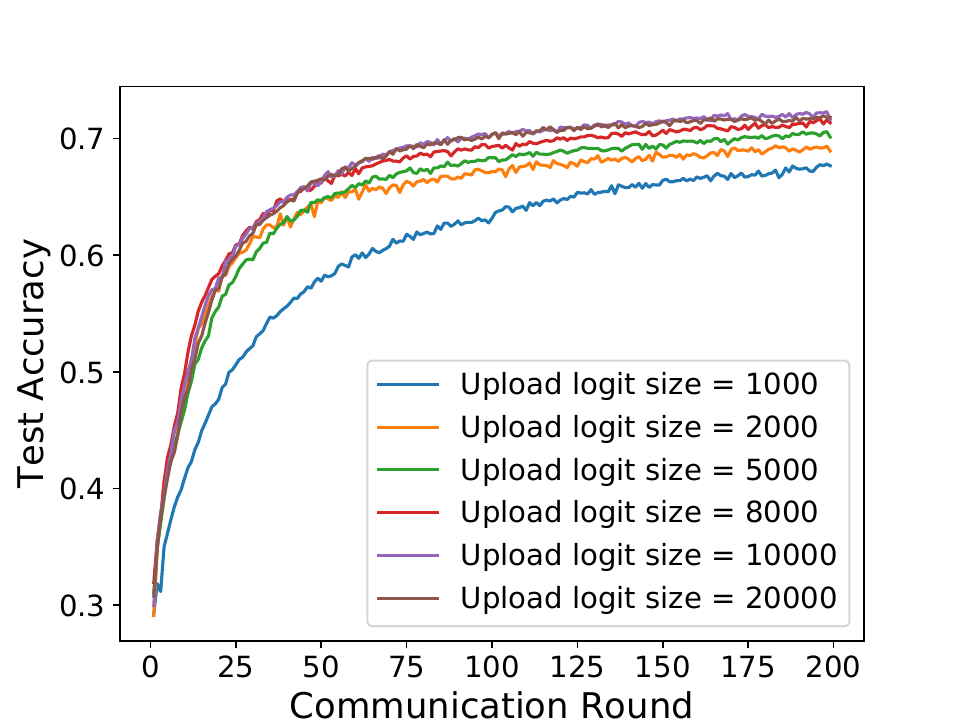}
		\caption{Size of upload logits.}
		\label{fig:uploadsize}
	\end{subfigure}
	\hfill
	\begin{subfigure}[b]{0.32\textwidth}
		\centering
		\includegraphics[width=\textwidth]{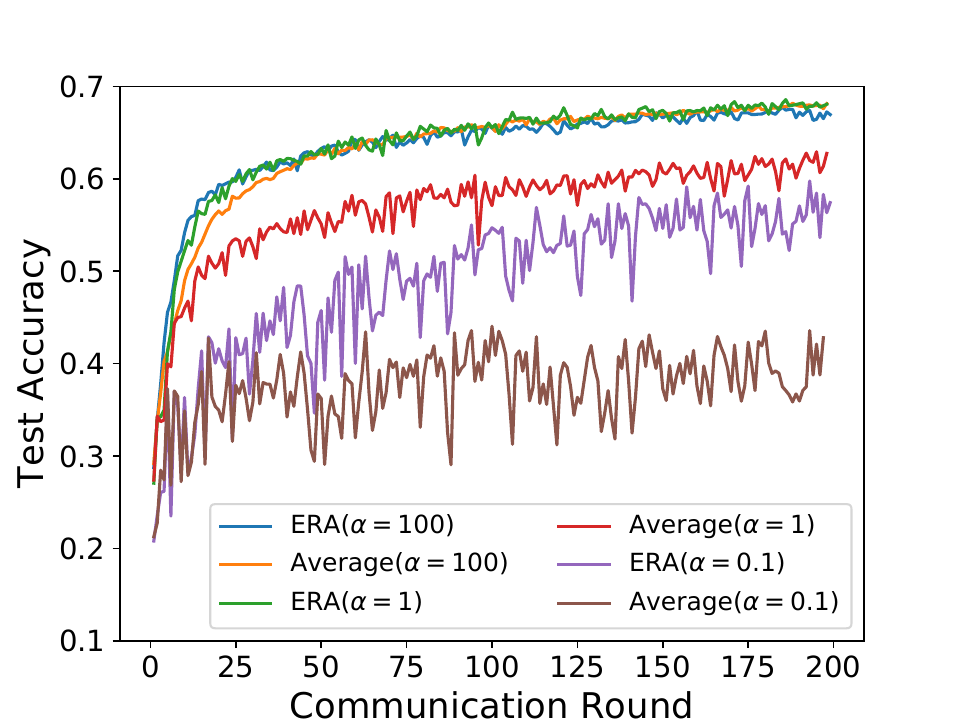}
		\caption{Aggregation method.}
		\label{fig:aggregation}
	\end{subfigure}
	\caption{Empirical observations of the algorithm key components' impacts on the FD training performance. The figure lists the test accuracy versus the number of communication rounds between the clients and server.}
	\label{fig:keycomponents}
\end{figure*}
	
	\subsubsection{Distillation dataset distribution}
	
	A vital assumption in FD is the availability of a public unlabeled dataset which enables the distillation process. In practice, it is not difficult to collect or generate many unlabeled samples. However, it is difficult to collect or generate a public dataset which has the same data distribution as the private labeled dataset. In the empirical simulations of the FD works, the public dataset distribution problem is often ignored. 
	
	To investigate the impact of the distillation dataset distribution, we performed experiments with two distillation datasets, i.e., CIFAR-10, the dataset with exactly the same distribution, and STL-10, the dataset with a similar but broader distribution. The result is demonstrated in Fig. \ref{fig:datadistribution}, which shows the test accuracy of the server model after the clients upload their logits or weights to the server, which is one communication round. It can be clearly seen that there exists a noticeable performance gap between the public dataset with similar distribution and the same distribution. And when distilling with CIFAR-10 dataset, the model reaches a comparable test accuracy with FedAvg. 
		
	\subsubsection{Upload Logits Size}
	
	In the FD meta-algorithm, the uplink communication cost is determined by the number of selected data samples, i.e., the upload logits size. A straightforward way to further reduce the communication cost is to reduce the size of the selected public dataset subset, $\mathcal{D}_{logit}$. However, this will cause insufficient data for the distillation step at the server side. Hence, there exists a trade-off between the communication cost and accuracy. 
	
	To empirically investigate this trade-off, we perform experiments where the size of the uploaded logits ranges from 1,000 to 20,000. 
	The empirical results are demonstrated in Fig. \ref{fig:uploadsize}.
	It is seen that increasing the upload logits size from a relatively small number improves the training performance. However, as the logits size increases to a very large number, the performance gain of more distillation data samples becomes marginal. For example, by uploading 20,000 logits, we barely see any performance gain compared with the one with 10,000.
	
	\subsubsection{Logits Aggregation Method}
	
	In the meta-algorithm, a simple average is adopted for the logits aggregation at the server side. However, the simple average method shows a bad performance when the local private data distribution is non-i.i.d.. Entropy reduction aggregation (ERA) is an aggregation method, which was proposed in \cite{dsflNishioMC}. There it was shown that it can achieve much better performance compared with the simple average method. 
	
	The main idea in ERA is to increase the confidence of the aggregated teacher logits during the server distillation step. ERA first averages the logits uploaded by the selected clients 
	\begin{equation}
		t (\boldsymbol{x}) = \frac{1}{| \mathcal{C} |} \sum{_{i \in \mathcal{C}}} t_i(\boldsymbol{x}) \quad \text{ for } \boldsymbol{x} \in \mathcal{D}_{logit},
	\end{equation}	
	The entropy of the averaged logit $t(\boldsymbol{x})$ is then reduced by:
	\begin{equation}
		\hat{t}(\boldsymbol{x}) = \frac{\exp{( t(\boldsymbol{x}) / T) }}{\sum_{m=1}^{N_c}\exp{(t(\boldsymbol{x})^m /T)}  }
	\end{equation}
	where $T$ here should be set between 0 and 1 so as to sharpen the output and reduce the entropy of $\hat{t}$.
	
	We adopt the Dirichlet distribution $Dir(\alpha)$ to simulate the non-i.i.d. data distribution in FL and  perform experiments with three levels of non-i.i.d. data distribution, i.e., $\alpha = 100, 1, \text{ and } 0.1$. It is noted that the data heterogeneity increases as $\alpha$ decreases. The result is demonstrated in Fig. \ref{fig:aggregation}. When $\alpha= 0.1$, i.e., the local data distribution is very non-i.i.d., reducing the entropy of the logits greatly improves the FD training performance.

	\subsection{Theoretical Verification}
	From the empirical observations, we have seen that for the FD meta-algorithm,
	\begin{enumerate}
		\item A public unlabeled dataset with the same input distribution is necessary to guarantee a good training performance;
		\item The size of upload logits influences the convergence speed.
 		More logits lead to a better performance, but the performance gain becomes marginal when there is a sufficient amount of uploaded logits;
		\item ERA improves the model performance of non-i.i.d. data distribution. 
	\end{enumerate}	
	
	In this subsection, we verify the latter two observations theoretically through a binary classification problem with Gaussian mixture models. Particularly, we show that for this setting, the FD meta-algorithm is equivalent to semi-supervised learning (SSL) with self-training \cite{pmlr-v130-oymak21a}.
	
	We first give a definition of the binary classification problem, the Gaussian mixture models, and self-training.
	For the binary classification problem, suppose there is a labeled dataset $\mathcal{S} = (\boldsymbol{x_i}, y_i) \in \mathbb{R}^p \times \{-1, +1 \}$ and $f: \mathbb{R}^p \rightarrow \mathbb{R}$ is the prediction function. The prediction rule is then defined as:
	\begin{equation}
		\hat{y}_f(\boldsymbol{x}) = 
		\begin{cases}
		1 \quad \text{if } f(x) \geq 0\\
		-1 \quad otherwise
		\end{cases}
	\end{equation} 
	\begin{definition}[Binary Gaussian Mixture Model (GMM)]
		The distribution $ (\boldsymbol{x}, y) \sim \mathcal{D}$ is given as follows. Fix a unit vector $\boldsymbol{u} \in \mathbb{R}^p$ and a scalar $\sigma >0$, and let $y$ be a Rademacher random variable ($\mathbb{P}(y=1) = 1- \mathbb{P}(y=-1) = \frac{1}{2}$ and $\boldsymbol{x} \sim \mathcal{N}(y\boldsymbol{u}, \sigma \boldsymbol{I_p})$).
		\label{definition}
	\end{definition}
	The component mean  $\boldsymbol{u}$ is the optimal linear classifier for this binary classification problem, where the prediction function is $f(\boldsymbol{x}) = \boldsymbol{u}^T\boldsymbol{x}$. With a labeled dataset $\mathcal{S}  = (\boldsymbol{x}_i, y_i)_{i=1}^n$, $\boldsymbol{u}$ can be estimated by the following averaging estimator
	\begin{equation}
	\boldsymbol{\beta}_{init} = \frac{1}{n}\sum_{i=1}^n y_i\boldsymbol{x}_i,
	\label{eq:10}
	\end{equation}
	The self-training approach uses the predicted labels $\hat{y}_f(\boldsymbol{x})$ for an unlabeled dataset $\mathcal{U} = \{ \boldsymbol{x}_i\}_{i=n+1}^{n+u} $ (a.k.a, pseudo labels) to self-train. 
	Given the initial averaging estimator $\boldsymbol{\beta}_{init}$ of the labeled dataset in \eqref{eq:10} and an acceptance threshold $\boldsymbol{\beta}_{init}^T\boldsymbol{x} >\Gamma$, the updated estimator after self-training with the unlabeled dataset $\mathcal{U}$ is then
	\begin{equation}
		\hat{\boldsymbol{\beta}} = \frac{\sum_{i=n+1}^u \mathds{1} ( |\boldsymbol{\beta}_{init}^T\boldsymbol{x}_i| >\Gamma ) sgn(\boldsymbol{\beta}_{init}^T\boldsymbol{x}_i) \boldsymbol{x}_i}{\sum_{i=n+1}^u \mathds{1} ( |\boldsymbol{\beta}_{init}^T\boldsymbol{x}_i| >\Gamma )}.
		\label{eq:11}
	\end{equation}

	
	In the following, we will show the training process of the FD meta-algorithm with the binary classification problem of GMM as the learning objective. In FD, there are $N$ locally stored private datasets, $\mathcal{S}^i = (\boldsymbol{x_i}^k, y_i^k)_{i=1}^{n_k}$, and the unlabeled auxiliary distillation dataset is denoted as $\mathcal{U} = \{\boldsymbol{x}_i\}_{i=n}^{n+u}$. Following the steps in the FD meta-algorithm (Algorithm \ref{algorithm1}), the training proceeds as follows 
	\begin{enumerate}
		\item \textbf{Local Update:} After the local updates, each user $k$ will have a local averaging estimator as 
		\begin{equation}
		\small
		\boldsymbol{\beta}_{init}^k = \frac{1}{n_k}\sum_{i=1}^n y_i\boldsymbol{x}_i^k
		\label{eq:9}
		\end{equation}
		
		\item \textbf{Logits Comptutation:} Each user will compute the local model output (logits) of the unlabeled dataset, i.e., $\{(\beta_{init}^k)^T\boldsymbol{x}_i\}_{i=n}^{n+u}$ and upload the logits to the server.
		
		\item \textbf{Logits Aggregation:} The server averages the logits and we have the averaged logits of the distillation dataset $\mathcal{U} = \{\boldsymbol{x}_i\}_{i=n}^{n+u}$ as
		\begin{equation}
		\sum_{k=1}^n \frac{n_k}{n} \{(\boldsymbol{\beta}_{init}^k)^T\boldsymbol{x}_i \} = \boldsymbol{\boldsymbol{\beta}}_{s}^T\boldsymbol{x}_i
		\label{eq:8}
		\end{equation}
		
		\item \textbf{Model Distillation:} The server creates pseudo labels by choosing data samples in $\mathcal{U}$ whose logits satisfy $|\boldsymbol{\beta}_{s}^T\boldsymbol{x}| > \Gamma$ and the pseudo labels are generated by $\tilde{y} = \hat{y}_{\beta_{s}^T\boldsymbol{x}}(\boldsymbol{x})$.
		
		After distillation with the averaged logits, the estimator at the server side with the averaged logits is then
		\begin{equation}
		\hat{\boldsymbol{\beta}} = \frac{\sum_{i=n+1}^u \mathds{1} ( |\boldsymbol{\beta}_{s}^T\boldsymbol{x}_i| >\Gamma ) sgn(\boldsymbol{\beta}_s^T\boldsymbol{x}_i) \boldsymbol{x}}{\sum_{i=n+1}^u \mathds{1} ( |\boldsymbol{\beta}_{s}^T\boldsymbol{x}_i| >\Gamma )}
		\label{eq:7}
		\end{equation}
		where $\Gamma >=0$ is the acceptance threshold that eliminates low-confidence predictions. It is noted that this is similar to the ERA method, which also eliminates high-entropy, i.e., low confidence predictions in the distillation process.
	\end{enumerate}
	We measure the estimator performance with the cotangent of the angle of the estimator $\boldsymbol{\beta}$ and the optimal classifier $\boldsymbol{u}$:
	\begin{equation}
	\small
		cot(\boldsymbol{\beta}, \boldsymbol{u}) = \frac{\rho(\boldsymbol{\beta}, \boldsymbol{u})}{\sqrt{1-\rho^2(\boldsymbol{\beta}, \boldsymbol{u})}}.
	\end{equation}

	With $\hat{\beta}$, we have the following theorem. 
	
	\begin{theorem} 
		(\hspace{-1.5mm} \cite{pmlr-v130-oymak21a})
		Let $\boldsymbol{u}\in \mathbb{R}^p$ be a uniform vector from Definition \ref{definition} and suppose $\beta_{s} \in \mathbb{R}^p$ as defined in  \eqref{eq:8} has correlation $\rho(\boldsymbol{\beta}_{s}, \boldsymbol{u}) = \alpha >0$. Set $\beta = \sqrt{1-\alpha^2}$ and draw i.i.d. unlabeled samples $\{\boldsymbol{x}_i\}_{i=n+1}^{n+u}$ from GMM. Let $\hat{\beta}$ be defined in \eqref{eq:7}. Define the normalized thresholds $\bar{\Gamma}_{\minus} = \frac{\alpha + \Gamma}{\sigma}$ and $\bar{\Gamma}_{+} = \frac{\Gamma - \alpha}{\sigma}$ and the quantities
		\begin{equation}
			\begin{split}
			\Lambda & = \frac{1}{2\pi\rho}( \exp(-\bar{\Gamma}_{+}^2 /2) + \exp(-\bar{\Gamma}_{\minus}^2 /2)) \\
			\rho & = Q(\bar{\Gamma}_{+}) + Q(\bar{\Gamma}_{\minus}) \\
			\nu = Q(\bar{\Gamma}_{\minus}) / \rho
			\end{split}
		\end{equation}
		where $Q(\cdot))$ is the tail of standard normal variable. Then, by fixing $\bar{u}= u /p$ and letting $p \rightarrow \infty$, we have
		\begin{equation}
			cot(\boldsymbol{\hat{\beta}}, \boldsymbol{\mu}) \xrightarrow[]{\mathbb{P}}  \frac{1+\sigma\alpha\Lambda-2\nu}{\sigma\sqrt{(1-\alpha^2)\Lambda^2+1/\bar{u}\rho }}.
		\end{equation}
		\label{theorem1}
	\end{theorem}
	
	\begin{proof}
		From \cite{pmlr-v130-oymak21a}, it is proved that for the self-training algorithm with the initial estimator $	\boldsymbol{\beta}_{init} = \frac{1}{n}\sum_{i=1}^n y_i\boldsymbol{x}_i$ in \eqref{eq:10} and $\hat{\boldsymbol{\beta}}$ in \eqref{eq:11} Theorem \ref{theorem1} holds.
		
		From \eqref{eq:9} and \eqref{eq:8}, 
		\begin{equation}
			\boldsymbol{\beta}_s = \sum_{k=1}^n \frac{n_k}{n} \boldsymbol{\beta}_{init}^k =  \frac{1}{n}\sum_{i=1}^n y_i\boldsymbol{x}_i.
		\end{equation}
		The estimator by distributed training of the labeled data samples is the same as the self-training algorithms. Thus, the result still holds for the FD meta-algorithm.
	 \end{proof}
	
	\begin{remark}
		Theorem \ref{theorem1} shows that for the GMM binary classification problem, the FD algorithm can obtain a higher correlation for the estimator than the initial estimator $\boldsymbol{\beta}_{init}$, i.e., a better model is obtained after the model distillation step. The distilled model $\hat{\boldsymbol{\beta}}$ benefits from a larger unlabeled dataset and a higher accepting threshold, which is consistent with the empirical observations in Section \ref{observations}.
	\end{remark}

	\section{Proposed Algorithm via Adaptive Data Sampling} \label{simulations}
	
	The theoretical and empirical results suggest two approaches to improve the training performance:
	 \begin{enumerate*}
	 	\item increase the size of the logits; or
	 	\item choose the public data with low-entropy logits.
	 \end{enumerate*}
	 Thus, we propose a communication-efficient FD algorithm with Active data Sampling (FAS).
	 To increase the size of the distillation logits while maintaining the communication cost, each user will actively select the low-entropy logits to be uploaded with its locally trained model.  
	 
	 The main difference between the FD meta-algorithm and the proposed FAS algorithm is the active data sampling step.
	 To generate better teacher logits, the entropy of the selected logits should be low, which means that the local model is confident. However, using the low entropy as the only criterion may lead to the result that every client is very confident about its uploaded logits, but the selected data samples for distillation are very easy to classify, which may degrade the final performance. This is similar to the process of human learning. If we always learn tasks that we are already very confident about, then we cannot learn new things. We need to learn something basic but we also need to explore new and challenging things. Thus, we propose the following mixed active data selection strategy.
	 For a selected client $i$, assuming the communication budget is $N_{logit}$ logit samples, then the active data sampling step proceeds as follows:
	 \begin{enumerate}
	 	\item Generate pseudo labels of $D_{pub}$ with the locally trained model $\theta_i$;
	 	\item Select $N_{logit}/2$ logits from $D_{pub}$ with an ascending order in entropy to generate  half of $D_{logit}^i$, and the pseudo label distribution in this half $\mathcal{D}_{logit}^i$ needs be close to the local data distribution;
	 	\item Randomly select $N_{logit}/2$ from $D_{pub}$ to generate the other half of $D_{logit}^i$.
	 \end{enumerate}
	 	 	
	We next provide experimental results to demonstrate the effectiveness of the proposed FAS algorithm. 
	In the simulated FD system, there are 20 clients. The local private training data are a subset with 20,000 data samples of the CIFAR-10 dataset, which means each user has only 1,000 local private data samples. The distillation dataset is the other 20,000 data samples of the CIFAR-10 dataset. The neural network model is ResNet-8. In each communication round, 8 clients are selected randomly to participate in the learning process. For the local update and distillation, we adopt Adam with a batch size of 8 as the optimizer. The local update steps and the distillation steps are set as 20 epochs in each communication round. The step sizes of the local update and distillation are set to 0.02 and 0.001, respectively. The step size decays at the 300-th and 600-th epochs by a rate of 0.1. 
	
    We compare the following 4 data sampling methods for FAS under different data distributions:
	\begin{enumerate}
		\item  No data sampling ($D_{logit}$ is the same);
		\item Random data sampling ($D_{logit}^i$ is randomly sampled from $D_{pub}$);
		\item Low-entropy data sampling ($D_{logit}^i$ is sampled from $D_{pub}$ assuming an entropy ascending order);
		\item Mixed-random-low-entropy sampling.
	\end{enumerate}
	
	The results are listed in Table \ref{table1} and Table \ref{table2}. To ensure that there is an overlap of the selected logits, we select $400$ logits from $2000$ unlabeled public data samples for the simulations in Table \ref{table1} and $2000$ logits from $8000$ unlabeled public data for the simulations in Table \ref{table2} in each communication round. We perform experiments with different degrees of non-i.i.d. data distribution controlled by $\alpha$ in the Dirichlet distribution. A smaller $\alpha$ leads to a more non-i.i.d. data distribution. 
	
	All the methods with data sampling exhibit better performance than the No data sampling method due to a larger distillation dataset. The performance of the random sampling method degrades evidently when the data becomes non-i.i.d., i.e., $\alpha$ decreases. And the performance of low-entropy sampling increases with more non-i.i.d. data.
	The mixed sampling strategy provides consistently better or comparable performance in terms of  test accuracy compared with the other sampling methods, under different degrees of non-i.i.d. local data distributions. 
		
	\begin{table}
		\caption{\small Test accuracy of ResNet-8 on CIFAR-10, $N_{logit} =500$.}	
		\centering
		\begin{tabular}{|c | c |c | c| c|} 
			\hline
			 & NoSample& Random & Low-Entropy & Mixed \\ [0.5ex] 
			\hline\hline
			$\alpha=100$ &0.6835 & 0.6956 & 0.6694 & \textbf{0.7058} \\ 
			\hline
			$\alpha=1$ &0.6376 & \textbf{0.6516} & 0.6468 & 0.6506 \\
			\hline
			$\alpha=0.1$ &0.4644 & 0.468 & 0.5219 & \textbf{0.5498} \\
			\hline
		\end{tabular}
	\label{table1}
	\end{table}
	
	\begin{table}
	\centering
	\caption{\small Test accuracy of ResNet-8 on CIFAR-10, $N_{logit} =2000$.}
	\begin{tabular}{|c | c |c | c| c|} 
		\hline
		 & NoSample& Random & Low-Entropy & Mixed \\ [0.5ex] 
		\hline\hline
		$\alpha=100$ &0.7314 & 0.7381 & 0.7031 & \textbf{0.7400} \\ 
		\hline
		$\alpha=1$ &0.6947 & 0.6827 & \textbf{0.7014} & 0.6947 \\
		\hline
		$\alpha=0.1$ &0.5511 & 0.5308 & 0.5573 & \textbf{0.5788} \\
		\hline
	\end{tabular}
	\label{table2}
	\end{table}

	\section{Conclusions}
	In this paper, we presented an FD meta-algorithm that incorporates existing FD methods and investigated the effects of key parameters both experimentally and theoretically to provide several design guidelines. Inspired by the design guidelines, a simple but effective FD algorithm with active data sampling was proposed. Experiments showed that the proposed algorithm performs consistently well under different distributions of heterogeneous data. Analyzing the FD meta-algorithm for neural networks and adapting it to the auxiliary dataset with similar distribution are left for future work.
	
	\bibliographystyle{IEEEtran}  
	\bibliography{FDref.bib}  
	
	

\end{document}